\documentclass[preprint,12pt, a4paper]{llncs}

\usepackage[utf8]{inputenc}
\usepackage{amsmath}
\usepackage{amssymb}
\usepackage{hyperref}
\usepackage{latexsym}
\usepackage{graphicx}
\usepackage{float}
\usepackage[dvipsnames]{xcolor}
\usepackage{multirow}
\usepackage{tabularx}
\usepackage{comment}
\usepackage{booktabs}
\usepackage{mdframed}
\usepackage{setspace}
\usepackage{tabularx}
\usepackage{mathtools}
\usepackage{array}
\usepackage[nohyperlinks,printonlyused,withpage]{acronym}
\usepackage{tikz}
\usetikzlibrary{decorations.pathreplacing,positioning, arrows.meta}
\usepackage{longtable}
\usepackage{pdflscape}
\usepackage{stmaryrd}
\usepackage{pdfpages}
\usepackage{color, colortbl}
\usepackage{amsfonts}
\usepackage{balance}
\usepackage{tcolorbox}
\usepackage{wrapfig}

\newcommand{\CN}[1]{\textcolor{red}{CN: #1}}
\newcommand{\LW}[1]{\textcolor{orange}{LW: #1}}
\newcommand{\TK}[1]{\textcolor{purple}{TK: #1}}

\newcolumntype{L}[1]{>{\raggedright\let\newline\\\arraybackslash\hspace{0pt}}m{#1}}
\newcolumntype{C}[1]{>{\centering\let\newline\\\arraybackslash\hspace{0pt}}m{#1}}
\newcolumntype{R}[1]{>{\raggedleft\let\newline\\\arraybackslash\hspace{0pt}}m{#1}}

\begin{document}

\title{On Scenario Formalisms for \\ Automated Driving}

\author{
	Christian Neurohr, Lukas Westhofen, Tjark Koopmann, Eike Möhlmann, Eckard Böde, Axel Hahn \\
	\email{\{firstname.lastname@dlr.de\}}
}

\institute{German Aerospace Center (DLR) e.V.,\\
	Institute of Systems Engineering for Future Mobility, \\ Oldenburg, Germany
}

\maketitle

\begin{abstract}
	The concept of scenario and its many qualifications -- specifically logical and abstract scenarios -- have emerged as a foundational element in safeguarding automated driving systems.
	However, the original linguistic definitions of the different scenario qualifications were often applied ambiguously, leading to a divergence between scenario description languages proposed or standardized in practice and their terminological foundation.
	This resulted in confusion about the unique features as well as strengths and weaknesses of logical and abstract scenarios. 
	To alleviate this, we give clear linguistic definitions for the scenario qualifications concrete, logical, and abstract scenario and propose generic, unifying formalisms using curves, mappings to sets of curves, and temporal logics, respectively. 
	We demonstrate that these formalisms allow pinpointing strengths and weaknesses precisely by comparing expressiveness, specification complexity, sampling, and monitoring of logical and abstract scenarios. 
	Our work hence enables the practitioner to comprehend the different scenario qualifications and identify a suitable formalism.
\end{abstract}

\section{Introduction}
\label{sec:introduction}

The scenario-based development \cite{nalic2020scenario,neurohr2020fundamental,iso21448} of automated driving systems (ADSs) at SAE Level $\ge 3$ \cite{sae2021definitions} complements the traditional safety processes of ISO 26262 \cite{iso26262}. 
In recent years, public research projects such as ENABLES3 \cite{leitner2020validation} as well as PEGASUS and VVM \footnote{\url{www.pegasusprojekt.de/en}, \url{www.vvm-projekt.de/en}} refined the scenario-based approach towards broader applicability. 

The focus of our work is the fundamental concept of scenario and its many qualifications.
The definitions for scene, situation, and scenario by Ulbrich et al.\ \cite{ulbrich2015defining} were established over time, with scene and scenario being standardized in ISO 21448 \cite{iso21448} and referenced in DIN SAE Spec 91381 \cite{dinsae2019} under slightly different formulations.
We hence use the standard notion of \emph{scenario} as a ``temporal development between several scenes in a sequence of scenes [...]'' where a \emph{scene} is ``a snapshot of the environment including the scenery and dynamic elements, as well as all actors’ and observers’ self-representations, and the relationships among those entities [...]'' \cite{ulbrich2015defining}. 

Many qualifications of the term scenario, i.e.\ subcategories, were introduced to account for usage in different development phases.
First, Menzel et al.\ \cite{menzel2018scenarios} introduced functional scenarios for the concept phase, logical scenarios to represent technical requirements, and concrete scenarios for testing. 
While this scenario qualification concept was widely accepted and used, their definitions are linguistic.
As the concept spread, so did the variety of implementations.
To standardize scenario description, ASAM OpenSCENARIO was developed \cite{asamOpenSCENARIOXML}. 
Initially, it was designed to specify concrete scenarios, with later versions (OpenSCENARIO XML 1.1 and upwards) expanding to logical scenarios.
To cope with the increasing complexity of operational design domains (ODDs), e.g.\ for urban driving, the need for more formal, logic-based scenarios arose \cite{neucrit21,foretellixMSDL,asamOpenSCENARIODSL}.
These abstract scenarios are situated between functional and logical scenarios, but employ a declarative instead of an imperative approach (as concrete and logical scenarios do).

Until now, the exact relationship between logical and abstract scenarios has not been explicated, leading to confusion about their respective strengths and weaknesses.
Adding onto this issue, practical implementations of scenario description languages sometimes deviated from the original definitions.
In order to close this gap, our work gives a long-due clarifying update of what concrete, logical, and abstract scenarios entail, informed by current practical usages.
We do so by proposing formal frameworks, i.e., \emph{scenario formalisms}, for the scenario qualifications.
Specifically, our work contributes
\begin{itemize}
	\item[(i)] a revision of the definitions of concrete, logical, and abstract scenarios establishing a formal basis for analysis,
	\item[(ii)] a comparison of logical and abstract scenarios w.r.t. expressiveness, specification complexity, sampling, and monitoring, and
	\item[(iii)] a guide for scenario-based practitioners for selecting an appropriate scenario qualification.
\end{itemize}

We start by examining the history of scenario descriptions in Section \ref{sec:history}, where related work is put in chronological order and contextualized by the authors' experience. 
Based on this, we sharpen the existing definitions for concrete, logical, and abstract scenarios by giving formal and linguistic definitions in Section \ref{sec:formalization}, guided by the practical implementations identified in Section \ref{sec:history}.
This forms the basis for a theoretical comparison of logical and abstract scenarios made in Section \ref{sec:comparison}, which highlights that our formal framework enables a precise differentiation between both.
Finally, the findings are discussed from a practitioner's perspective in Section \ref{sec:discussion}.

\section{History of Scenario Qualifications}
\label{sec:history}

We start by unraveling the genesis of the different qualifications of concrete, logical, and abstract scenarios in the context of automated driving, which also highlights related work.
As we are particularly interested in clarifying the boundary between logical and abstract scenarios, we examine various proposal of scenario description languages for both to establish a well-informed foundation.

\subsection{Concrete, Logical, and Functional Scenarios}
\label{subsec:history_scenario_qualifications}

While the linguistic definitions of scene and scenario are suitable for human comprehension, they need to be refined depending on the stage of the development process.
For this, Menzel et al. \cite{menzel2018scenarios} introduce three scenario qualifications: concrete, logical, and functional scenarios.
For completeness, we revisit their original linguistic definitions. 
	\begin{definition}[Concrete Scenario \cite{menzel2018scenarios}]
		``Concrete scenarios distinctly depict operating scenarios on a state space level. Concrete scenarios
		represent entities and the relations of those entities
		with the help of concrete values for each parameter
		in the state space.''
	\end{definition}
Note that from this definition it is not clear whether a concrete scenario merely specifies an operating scenario's starting conditions or whether it specifies a fixed sequence of scenes. We may assume the former, as concrete scenarios are to be used in the testing phase according to Menzel et al. In Section \ref{sec:formalization} we entangle this confusion by differentiating between (trajectory-level) concrete scenarios and attribute-level concrete scenarios.
	\begin{definition}[Logical Scenario \cite{menzel2018scenarios}]
		``Logical scenarios include operating scenarios on
		a state space level. Logical scenarios represent the
		entities and the relations of those entities with the
		help of parameter ranges in the state space. The
		 parameter ranges can optionally be specified with probability distributions. Additionally, the relations
		of the parameter ranges can optionally be specified
		with the help of correlations or numeric conditions.
		A logical scenario includes a formal notation of the
		scenario.''
	\end{definition}
In the development process, logical scenarios should be used at the level of technical requirements where a formal, machine-readable, and imperative scenario specification is needed. 
	\begin{definition}[Functional Scenario \cite{menzel2018scenarios}]
		``Functional scenarios include operating scenarios
		on a semantic level. The entities of the domain
		and the relations of those entities are described via
		a linguistic scenario notation. The scenarios are
		consistent. The vocabulary used for the description of functional scenarios is specific for the use case
		and the domain and can feature different levels of
		detail.''
	\end{definition}
Functional scenarios shall be used in early concept phases, e.g.\ for hazard analysis and risk assessment or stakeholder alignment.

When researchers in industry and academia started to put these concepts to practice, their linguistic definitions led to ambiguous usages.
Functional scenarios emerged as informal scenario descriptions with a strong focus on readability and visualization for humans \cite{graubohm2020towards}.
At the other end of the spectrum, concrete scenarios gained traction when OpenSCENARIO, a machine-readable, imperative language, was advertised by VIRES within the projects PEGASUS/ENABLE-S3 \cite{viresOpenSCENARIO}, which was later transferred to ASAM for standardization \cite{asamOpenSCENARIOXML}.
In order to have a bottom-up generalization of concrete scenarios based on e.g. real-world data, scenario-based practitioners abstracted from fixed parameter values to parameter ranges, introducing degrees of freedom.
Sometimes, parameter ranges are accompanied by probability distributions. 
In this way, the instantiation of logical scenarios is reduced to sampling from parameter ranges (according to their distributions) prior to execution.

Examples usages of logical scenarios include the assurance of automated driving on highways \cite{weber2019framework}, performance evaluation of ADSs \cite{zhang2022performance}, or clustering of scenarios from real-world data, among many others, cf.~\cite[Section~3]{neurohr2020fundamental}.
OpenSCENARIO V1.1 extended the standard to support logical scenarios in the same way.
This was used by a car manufacturer to specify logical scenarios for an automated lane keeping system\footnote{\url{www.github.com/openMSL/sl-3-1-osc-alks-scenarios}}, derived from the UN Regulation No. 157 \cite{unece157}.

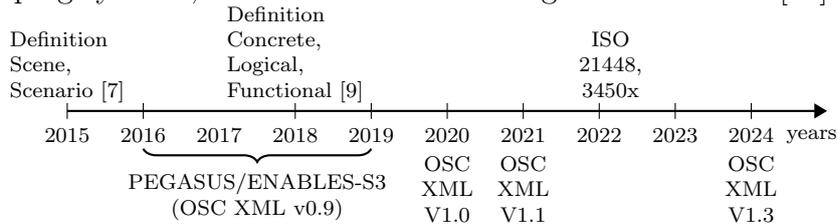
\begin{figure}
	\centering

\begin{tikzpicture}
	\draw[thick, -Triangle] (0,0) -- (10,0) node[font=\scriptsize,below left=3pt and -8pt]{years};
	
	\foreach \x in {0,1,3,4,5,6,7,8,9}
	\draw (\x cm,3pt) -- (\x cm,-3pt);
	
	\foreach \x/\descr in {0/2015, 1/2016, 2/2017, 3/2018, 4/2019, 5/2020, 6/2021, 7/2022, 8/2023, 9/2024}
	\node[font=\scriptsize, text height=1.75ex,
	text depth=.5ex] at (\x,-.3) {$\descr$};
	
	
	\node[align=left, font=\scriptsize, text height=1.75ex,
	text depth=.5ex] at (0,0.3) {Definition \\ Scene, \\ Scenario \!\! \cite{ulbrich2015defining}};

	\node[align=left, font=\scriptsize, text height=1.75ex,
	text depth=.5ex] at (3,0.3) {Definition \\ Concrete, \\ Logical, \\ Functional \!\! \cite{menzel2018scenarios}};

	\node[align=center, font=\scriptsize, text height=1.75ex,
	text depth=.5ex] at (5,-1.35) {OSC \\ XML \\ V1.0 };
	
	\node[align=center, font=\scriptsize, text height=1.75ex,
	text depth=.5ex] at (6,-1.35) {OSC \\ XML \\ V1.1 };
	
	\node[align=center, font=\scriptsize, text height=1.75ex,
	text depth=.5ex] at (7.15,0.3) {ISO \\ 21448, \\ 3450x};
	
	
	\node[align=center, font=\scriptsize, text height=1.75ex,
	text depth=.5ex] at (9,-1.35) {OSC \\ XML \\ V1.3 };
	
	\draw [thick,decorate,decoration={brace,amplitude=5pt}, align=center] (4,-.5) -- +(-3,0)
	node [black,midway,font=\scriptsize, below=5pt] {PEGASUS/ENABLES-S3 \\ (OSC XML v0.9)};
\end{tikzpicture}
	\vspace{-0.5cm}
	\caption{Time of scenario description activities.}
	\label{fig:timeline_scenarios_description}
\end{figure}

Note that the definition of logical scenarios allows for relations on parameter ranges, e.g.\ $v_{\text{actor}_1} < v_{\text{actor}_2}$.
As this has not been recognized in practice, it was not standardized. 
As can be seen on the time of \autoref{fig:timeline_scenarios_description}, OpenSCENARIO XML is being actively developed. 

We conclude that the emergence of concrete scenarios followed from the need of specifying traffic scenarios in simulation, particularly for testing. 
Logical scenarios suggested themselves as a straight-forward extension, enabling the sampling of concrete scenarios.

\subsection{Abstract Scenarios}
\label{subsec:history_formal}

Parallel to the history of the technology-driven scenario descriptions introduced above -- largely fueled by practical needs -- lies the history of declarative traffic scenario specification based on logics. 

In 2011, Hilscher et al.\ \cite{hilscher2011abstract} presented the \emph{Multi-Lane Spatial Logic} (MLSL) for formally proving safety of multi-lane highway traffic, based on Interval Temporal Logic and Duration Calculus.
MLSL deliberately does not consider vehicle dynamics.
Both liveness \cite{Schwammberger18a} and fairness \cite{Schwammberger2019fair} were examined. 
Although the logic itself is undecidable, there is implementation work under some finiteness assumptions \cite{Fraenzle2015}.

A different approach based on visualization, called \emph{Visual Logic} (VL), was introduced by Kemper and Etzien in 2014 \cite{kemper2014visual} to specify sequences of highway traffic situations for driver assistance systems. 
It was designed for communicating about scenarios in an interdisciplinary setting, e.g.\ between systems engineers, traffic psychologists, and safety engineers. 
Besides communication, VL focused on spatio-temporal logics with \emph{monitoring} as a potential use case.

In 2017, using VL as inspiration, Damm et al.\ introduced Traffic Sequence Charts (TSCs) as a visual formalism for the declarative specification of scenarios \cite{damm2017tsc,damm2018formal}. 
TSCs use multi-sorted first-order real-time logic by interpreting formulae over a world model. 
In contrast to MLSL, TSCs consider vehicle dynamics.
The objective of visual interdisciplinary communication was extended by requirement specification, scenario catalogs specification, and sampling.
TSCs were subsequently considered for e.g.\ consistency \cite{becker2020partial}, play-out \cite{becker2022simulation}, in the development process \cite{damm2018using}, runtime monitoring \cite{grundt2022rm}, knowledge formalization \cite{borchers2024using}, and a virtual verification toolchain \cite{borchers_tsc2carla_2025}.

Subsequently, an automotive tool vendor started to define abstract scenarios using a constraint language (called Traffic Phases from now on) \cite{eggers2018constraint}. 
They translate scenarios into non-linear SMT problems \cite{scheibler2019solving}, which can subsequently be solved algorithmically.

Shortly afterwards, Foretellix released the \emph{Measurable Scenario Description Language} (M-SDL) \cite{foretellixMSDL}, a ``mostly declarative programming language'' that focuses on the specification of abstract scenarios for sampling and monitoring. 
Around the same time, Fremont et al. introduced \emph{Scenic} \cite{fremont2019scenic,fremont2023scenic}, which is a probabilistic programming language for specifying declarative constraints on scene-level. 
For their main use case -- sampling of scenarios -- they rely on generating scenarios scene by scene with probabilistic models, checking feasibility of the current state in every step. 
In contrast to SMT-based sampling approaches, Scenic is highly dependent on simulation models.

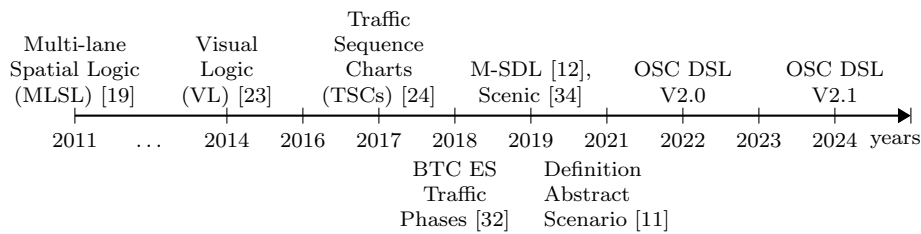
\begin{figure}
	\vspace{0.25cm}
	\centering

\begin{tikzpicture}
	\draw[thick, -Triangle] (-1,0) -- (10,0) node[font=\scriptsize,below left=3pt and -8pt]{years};
	
	\foreach \x in {-1,1,2,3,4,5,6,7,8,9,10}
	\draw (\x cm,3pt) -- (\x cm,-3pt);
	
	\foreach \x/\descr in {-1/2011, 01/2014, 2/2016, 3/2017, 4/2018, 5/2019, 6/2021, 7/2022, 8/2023, 9/2024}
	\node[font=\scriptsize, text height=1.75ex,
	text depth=.5ex] at (\x,-.3) {$\descr$};
	
	\node[font=\scriptsize, text height=1.75ex,
	text depth=.5ex] at (0,-.3) {$\dots$};
	
	\node[align=center, font=\scriptsize, text height=1.75ex,
	text depth=.5ex] at (-1,.3) {Multi-lane \\ Spatial Logic \\ (MLSL) \cite{hilscher2011abstract}};
	
	\node[align=center, font=\scriptsize, text height=1.75ex,
	text depth=.5ex] at (1,.3) {Visual \\ Logic \\ (VL) \cite{kemper2014visual}};
	
	
	\node[align=center, font=\scriptsize, text height=1.75ex,
	text depth=.5ex] at (3,.3) {Traffic \\ Sequence \\ Charts  \\ (TSCs) \cite{damm2017tsc}};
	
	\node[align=center, font=\scriptsize, text height=1.75ex,
	text depth=.5ex] at (4,-1.35) {BTC ES \\ Traffic \\ Phases \cite{eggers2018constraint}};
	
	\node[align=left, font=\scriptsize, text height=1.75ex,
	text depth=.5ex] at (6.0,-1.35) {Definition \\ Abstract \\ Scenario \cite{neucrit21}};
	
	\node[align=center, font=\scriptsize, text height=1.75ex,
	text depth=.5ex] at (5,.3) {M-SDL \cite{foretellixMSDL}, \\ Scenic \cite{fremont2019scenic}};
	
	\node[align=center, font=\scriptsize, text height=1.75ex,
	text depth=.5ex] at (7,.3) {OSC DSL \\ V2.0 };
	
	\node[align=center, font=\scriptsize, text height=1.75ex,
	text depth=.5ex] at (9,.3) {OSC DSL \\ V2.1 };
\end{tikzpicture}
	\vspace{-0.75cm}
	\caption{Timeline of abstract scenario specification languages.}
	\label{fig:timeline_abstract_scenarios}
\end{figure}
In 2021, Neurohr et al. provided the first linguistic definition of the term abstract scenario.
\begin{definition}[original, cf.~\cite{neucrit21}]
	``An abstract scenario
	is a formalized, declarative description of a traffic scenario focusing on complex relations, particularly on causal
	relations.''
	\label{def:abstract_scenario_original}
\end{definition}
They related the abstract scenario with the scenario qualifications of Section \ref{sec:history} -- situated between the functional and the logical scenario regarding their abstraction level \cite[Figure~14]{neucrit21}.

At this point, a shared industry interest in declarative scenario specification languages had developed. 
In 2022, this converged to the OpenSCENARIO DSL standard, a domain-specific language with a custom temporal logic over dense finite traces with a strong focus on abstract scenario specification. 
Notably, it adopts \autoref{def:abstract_scenario_original}. 

\section{Definition and Formalization of Scenario Qualifications}
\label{sec:formalization}

As Section \ref{sec:history} shows, various mechanisms for specifying scenarios arose over time, which were, however, not aligned.
To clearly compare and explicate strengths and weaknesses of the mechanisms, we now present a \emph{unifying formal framework} for concrete, logical, and abstract scenarios. 
For each qualification, we also derive an updated linguistic definition.
Note that functional scenarios are natural-language specifications and not considered further.
We start with concrete scenarios as the unifying foundation of logical and abstract scenarios.

\subsection{Concrete Scenarios}
\label{subsec:construction_concrete_scenarios}

We start with a bottom-up formalization of concrete scenarios based on \emph{scenes}.
Moreover, we restrict to continuous time here, although we acknowledge discrete time as an important facet as well.
We also deliberately avoid specifying which elements are contained in a scene.

\begin{remark}[Continuity on Sets of Scenes]
	To disregard invalid scenarios, we require a notion of continuity on sets of scenes.
	Therefore, we assume that there exists a metric $d : \mathcal{S} \times \mathcal{S} \rightarrow \mathbb R$ such that $(\mathcal{S},d)$ is a metric space where
	continuity can be defined via the $\varepsilon$-$\delta$-criterion.
	\label{remark:scenes_continuity}
\end{remark}

\begin{example}
	Consider that time-dependent variables in a scene are positions, velocities, and accelerations of actors. Then, the set of scenes restricted to a road network and a fixed set of actors together with the Euclidean distance becomes a metric space.
\end{example}

\begin{definition}[Concrete Scenario]
\label{def:concrete_scenario}
For a connected time domain $T$ a \emph{concrete scenario} is a piecewise continuous map $$C \colon T \rightarrow \mathcal{S}, t \mapsto C(t)$$ from $T$ to the set of scenes $S$, with starting scene $S_0 = C(0)$. 
The set of all concrete scenarios over $T$ is denoted by $\mathcal{C}^T$. 
$\mathcal{C}$ describes the set of all concrete scenarios, independent of their time domain.
\end{definition}

From this definition, a linguistic version is easily derived:
\begin{tcolorbox}[colback=white, boxrule=0.5pt]
	Given a starting scene and a time domain, a \emph{concrete scenario} is a \emph{seamless assignment} of \emph{points in time} to \emph{scenes}.
\end{tcolorbox}

Essentially, concrete scenarios are continuous extensions of a time series of real-world measurements or of logged simulation data within a limited duration. In a sense, they are \emph{post-execution} scenarios -- where the executioner could refer to traffic reality, a proving ground set up, or a simulator. In the context of OpenSCENARIO DSL, these are also known as \emph{trajectory-level concrete scenarios}.

However, for applications such as scenario-based testing, practitioners may want to specify concrete test cases independent of the execution of the system under test (SuT) and other involved actors.
This requires a notion of \emph{pre-execution} concrete scenario, for which we adopt the term \emph{attribute-level concrete scenarios} used by OpenSCENARIO DSL. 
We start by introducing \emph{deterministic models}.

\begin{definition}[Deterministic Model]
	\label{def:det-model}
A \emph{deterministic model} is a dynamical system $(\Theta, \mathcal{S}, \varphi)$ with time domain $\Theta$, $0 \in \Theta$, and state space $\mathcal{S}$, the set of all scenes, together with a piecewise continuous evolution function $\varphi \colon  \Theta \times \mathcal{S} \rightarrow \mathcal{S}$ such that 
	\begin{align}
		\begin{split}
			\label{eq:deterministic_model}
			\varphi(0, S_0) & = S_0 \in \mathcal{S} \text{ and } \\ \varphi(\theta_2, \varphi(\theta_1, S)) & = \varphi(\theta_1 + \theta_2, S).
		\end{split}
	\end{align} 
\end{definition}

\begin{tcolorbox}[colback=white, boxrule=0.5pt]
	 A \emph{deterministic model} for an actor assigns a \emph{history of scenes} (i.e.\ a concrete scenario) to an element of the \emph{state space} of that actor.
\end{tcolorbox}

\begin{remark}
Note that a finite family of deterministic models $\Phi = ((\Theta_i, \mathcal{S}, \varphi_i))_{i=1}^N$, $N \in \mathbb{N}$, can again be understood as a deterministic model with time domain $\bigcap_{i = 1}^N \Theta_i$ if the models do not contradict each other.
In case the models $\varphi_i$ and $\varphi_j$ contradict each other at time $t > 0$, replace $\Theta_i$ by $\Theta_i \cap (-\infty, t - \epsilon]$, $\epsilon > 0$ small.
Since the models do not contradict at time $0$ by definition, the smallest possible time domain is $\{0\}$.
\end{remark}

In this work, we assume time domains of the form $T(\Phi,S_0) = [0,t_{\sup}(\Phi,S_0))  \subset \mathbb R$, where $t_{sup}(\Phi, S_0)$ may be infinite, or a discretization of $T(\Phi, S_0)$.
The set-valued function $T$ assigns to each combination of a family of models $\Phi$ and a starting scene $S_0$ the minimal time domain that is required for each specified model to reach one of their end states from the starting scene.
Whenever a model in $\Phi$ does not reach one of their end states we have that $T(\Phi, S_0) = [0, \infty)$.
For certain applications the time domain is given and the models are calibrated to conform to it.
It is then assumed that every model is defined for at least the specified time domain and $t_{sup}$ is constant, i.e. $t_{sup}(\cdot, \cdot) \equiv c \in \mathbb{R} \cup \{\infty\}$.
In these cases we may identify the constant function $T(\cdot, \cdot) \equiv [0, c]$ or $T(\cdot, \cdot) \equiv [0, \infty)$ with the interval $[0, c]$ or $[0, \infty)$, respectively, and refer to it simply as $T$.

Based on \autoref{def:det-model}, we can now define attribute-level concrete scenarios (or \emph{pre-execution concrete scenarios}). Note that this is a first step abstracting away from (post-execution) concrete scenarios.

\begin{definition}[Attribute-level Concrete Scenarios]
\label{def:attr-level-con-scen}
	Let $\mathcal{M}$ be the set of deterministic models, $T(\cdot, \cdot)$ a time domain assignment, and $S_0$ a starting scene. An \emph{attribute-level concrete scenario} is a family of maps
	\begin{equation}
		(\alpha^\Phi_{S_0})_{\Phi \in [\mathcal{M}]^{<\omega}} \quad \text{with} \quad \alpha^\Phi_{S_0} \colon T(\Phi,S_0) \rightarrow \mathcal{S}, t \mapsto \Phi(t, S_0)\,,
	\label{eq:attribute_level_concrete_scenario}
	\end{equation}
	where $[\mathcal{M}]^{<\omega}$ denotes the finite subsets of $\mathcal{M}$.
	By construction, each family member is a concrete scenario in the sense of \autoref{def:concrete_scenario}.
	The set of all attribute-level concrete scenarios is denoted by $\mathcal{C}_\mathit{attr}$. 
\end{definition}
	\begin{tcolorbox}[colback=white, boxrule=0.5pt]
	An \emph{attribute-level concrete scenario} is an imperative specification from which, given a fixed set of deterministic models, a unique concrete scenario can be determined.
\end{tcolorbox}

\begin{example}
		\label{ex:concrete-scenario}
		Consider scenes that contain the 2-dimensional position and velocity of one actor $A$, i.e.\ $\mathcal{S} = \mathbb{R}^{4}$ over a time domain $[0,20]$, with initial position $p(0) = (-50, 100)^{\mathsf{T}}$ and constant velocity vector  $v(t) = (10,-5)^{\mathsf{T}}$. Then,
		\begin{equation}
			\label{eq:concrete_scenario_example}
			C \colon [0,20] \rightarrow \mathbb{R}^{4}, t \mapsto (-50 + 10 \cdot t, 100 - 5\cdot t, 10, - 5)^{\mathsf{T}}
		\end{equation}
		is a concrete scenario with starting scene $S_0 = (-50,100,10,-5)^{\mathsf{T}}$.
	Keeping the deterministic model variable and taking $T(\Phi,S_0)=[0,20]$ as constant, we obtain an attribute-level concrete scenario 
	\begin{equation*}
			\alpha^\Phi_{S_0} \colon [0,20] \rightarrow \mathcal{S}, t \mapsto \Phi(t, S_0)\,,
	\end{equation*}
	where $\Phi \in \mathcal{M}$ is any model defined on $T=[0,20]$ and $\mathcal{S} = \mathbb{R}^{4}$. Then, for
	 the deterministic model
	\begin{equation*}
		\varphi \colon [0,\infty) \times \mathbb{R}^4 \rightarrow \mathbb{R}^4, \left(t,
		\begin{pmatrix}
			s_0 \\
			s_1 \\
			s_2 \\
			s_3 
		\end{pmatrix}\right)
		\mapsto \begin{pmatrix}
			s_0 + 10 \cdot t \\
			s_1 - 5\cdot t  \\
			s_2 \\
			s_3 
		\end{pmatrix}
	\end{equation*}
	we have that $\alpha^\varphi_{S_0} = C$ is the concrete scenario from equation \eqref{eq:concrete_scenario_example}.
\end{example}

\begin{remark}
	Note that attribute-level concrete scenarios abstract from concrete scenarios by keeping models variable.
	It merely specifies the pre-execution conditions in form of the starting scene.
	When models are fixed, the concrete scenario is then dynamically created by letting the models run.
	The piecewise continuity follows from the piecewice continuity of deterministic models required by \autoref{def:det-model}.
\end{remark}

\subsection{Logical Scenarios}
\label{subsec:construction_logical_scenarios}
We previously constructed concrete scenarios bottom-up by using scenes.
Similarly, we now define (attribute-level) logical scenarios.

\begin{definition}[(Attribute-level) Logical Scenario]
	\label{def:logical_scenario}
	Let $\mathcal{M}$ be the set of deterministic models. 
	An \emph{attribute-level logical scenario} is a map
	\begin{align}
		\begin{split}
			\label{eq:logical_scenario}
			 L_{\mathcal{M}} \colon X \rightarrow \mathcal{C}_{attr}\,,
			 x \mapsto \left(\alpha_{S_0(x)}^{\Phi(x)}\right)_{\Phi \in [\mathcal{M}]^{<\omega}}
		\end{split}
	\end{align}
	with parameter space $X \subset \mathbb{R}^n$ of dimension $n \in \mathbb N$ where $x \in X$ configures the starting scene $S_0^x$ and the models $\Phi^x$.
	By construction, $L_{\mathcal{M}}(x)$ is a attribute-level concrete scenario for each $x \in X$.
	
	Analogously, for fixed $\Phi \in [\mathcal{M}]^{<\omega}$, a \emph{logical scenario} is a map
	\begin{equation*}
		 L_{\Phi} \colon X \rightarrow \mathcal{C}\,,
		x \mapsto \alpha_{S_0(x)}^{\Phi(x)}\,.
	\end{equation*}
	By construction, $L_{\Phi}(x) \in \mathcal{C}$ is a concrete scenario for all $x \in X$.
\end{definition}

From this formal definition we easily derive a simple linguistic definition of logical scenarios:
\begin{tcolorbox}[colback=white, boxrule=0.5pt]
	An \emph{attribute-level logical scenario} is an assignment of elements of a \emph{parameter space} to \emph{attribute-level concrete scenarios}. 
	\\ \\
	For a fixed set of deterministic models, a \emph{logical scenario} is an assignment of elements of a parameter space to \emph{concrete scenarios}.
\end{tcolorbox}
	
\begin{remark}[Parameter Distributions in Logical Scenarios]
	In practice, the specification of logical scenarios often entail a probability distribution over the parameter space. We abstain from this here, but instead discuss it in the context of sampling Section \ref{subsec:comparison_sampling}.
\end{remark}

\subsection{Abstract Scenarios}
\label{subsec:construction_abstract_scenarios}		

Abstract scenarios are a formal and declarative description constraining traffic happenings \cite{neucrit21}, which is typically done by using logics.
As we are interested in constraining concrete scenarios, we first define what properties a suitable logic for abstract scenario shall possess.
For this, it is natural to extend the ideas of \autoref{def:det-model} and \autoref{def:attr-level-con-scen} to temporal logics.
Intuitively, a scenario logic constrains the set of possible concrete scenarios according to some formula given a finite history and a time point.
Note that the concrete scenarios must not necessarily be of the same duration.
\begin{definition}[Scenario Logic]
	\label{def:scenario-logic}
	Let $T$ be an unbounded time domain, $\mathcal{S}$ the set of all scenes, and $\mathcal{C}$ all concrete scenarios independent of their time domain.
	$T^C$ denotes the time domain of $C \in \mathcal{C}$.
	We write $C \subseteq C'$ if $C$ is a prefix of $C'$.
	A \emph{scenario logic} is a finite temporal logic $\Lambda$ with denotational semantics
	\begin{equation*}
		\llbracket \cdot, \cdot, \cdot \rrbracket \colon \Lambda \times \mathcal{C} \times T \rightarrow 2^\mathcal{C}\,.
	\end{equation*}
	for which holds that for all $\lambda \in \Lambda$, $C \in \mathcal{C}$, and $t, t_1, t_2 \in T$
	\begin{align*}
		\llbracket \lambda, C, 0 \rrbracket & = \{C\}\text{,}\\
		\llbracket \lambda, C, t_1 + t_2 \rrbracket & = \bigcup_{C_1 \in \llbracket \lambda, C, t_1 \rrbracket} \llbracket \lambda, C_1, t_2 \rrbracket \text{, and}\\
		\forall C' \in \llbracket \lambda, C, t \rrbracket &\colon\ C \subseteq C'
	\end{align*}
	For any $\lambda \in \Lambda$, $C \in \mathcal{C}$, and $t \in T$ we demand that any concrete scenario in $\llbracket \lambda, C, t \rrbracket$ is piecewise continuous.
	Moreover, we require $\Lambda$ supporting the operation $\lambda_1 \wedge \lambda_2 \in \Lambda$ with $\llbracket \lambda_1 \wedge \lambda_2, C, t \rrbracket \coloneq \llbracket \lambda_1, C, t \rrbracket \cap \llbracket \lambda_2, C, t \rrbracket$. 
	The set of concrete scenarios of a formula $\lambda \in \Lambda$ is defined as
	$\mathcal{C}(\lambda) = \{ C \in \mathcal{C} \mid \forall t \in T\colon \llbracket \lambda, C, t \rrbracket = \{ C \}\} \cup \{ C \in \mathcal{C} \mid \forall t \in T^C\colon\ \exists t' \in T\colon C_{<t+t'} \in \llbracket \lambda, C_{<t}, t' \rrbracket \}$ (the first defining bounded and the second set unbounded concrete scenarios).
\end{definition}
$\llbracket \cdot, \cdot, \cdot \rrbracket$ represents a branching semantics, i.e.\ it allows the logic to create \emph{trees} of concrete scenarios (or, in other words, allows branching at each point in time, based on the scenario history).
Moreover, there is a natural connection to \autoref{def:det-model}: 
Intuitively, a formula $\lambda$ is a deterministic model if for all $C \in \mathcal{C}$ and $t \in T$, $|\llbracket \lambda, C, t \rrbracket| \leq 1$.

Abstract scenarios are simply formulae in a scenario logic, which are split up in a world model (a set of models of actor behaviors) and the scenario constraints (restricting the world model further).

\begin{definition}[Abstract Scenario]
	\label{def:abstract-scenario_new}
	Let $T$ be a time domain, $\Lambda$ a scenario logic over $T$, and $\mathcal{S}$ the set of all scenes.
	An \emph{abstract scenario} is a tuple $A = (\lambda, \omega_1, \dots, \omega_n) \in \Lambda^{n+1}$, where $\lambda$ are the scenario constraints and $\omega_1, \dots, \omega_n$ form a world model.
	Its set of concrete scenarios is defined as $\mathcal{C}(A) = \mathcal{C}(\lambda \wedge \bigwedge_{i=1}^n \omega_i)$.
\end{definition}

The formal definitions given prior can be understood in the following, intuitive way, sharpening \autoref{def:abstract_scenario_original}:
\begin{tcolorbox}[colback=white, boxrule=0.5pt]
	An \emph{abstract scenario} is a \emph{declarative} specification of a set of concrete scenarios that uses a \emph{temporal logic} to restrict the set of all possible scenarios as described by a world model.
\end{tcolorbox}
Note that the focus on ``complex relations'' has been dropped compared to \autoref{def:abstract_scenario_original}, as the term, even from a practical perspective, remained rather vague.
However, two essential aspects are brought out: We require a temporal logic and a declarative semantics.
Both can be found in \autoref{def:scenario-logic} -- having $T$ in the domain of $\llbracket \cdot, \cdot, \cdot \rrbracket$ and using a denotational semantics by restricting $2^\mathcal{C}$.

\begin{example}
	Consider the set of scenes $\mathcal{S} = \mathbb{R}^4$ used in \autoref{ex:concrete-scenario} and $T = [0,\infty)$.
	Our example scenario logic allows for constants $\mathbb{R}^4$, a unary operator $\lozenge$ (eventually), and binary operators $\vee$ and $\wedge$, i.e.\ it adheres to the grammar $\lambda \Coloneqq \mathbb{R}^4 \mid \lozenge \lambda \mid \lambda \vee \lambda \mid \lambda \wedge \lambda$.
	For $x \in \mathbb{R}^4$, $t \in [0, \infty)$, and a concrete scenario $C\colon T^C \to \mathbb{R}^4$, their semantics is
	\begin{align*}
		\llbracket \lambda, C, t \rrbracket &= \{ C \}  \text{ if } t_\mathit{max}^{T^C} \geq 20 \text{, and otherwise}\\
		\llbracket x, C, t \rrbracket &= 
		\begin{cases}
			\{ C \} & \text{ if } t = 0 \text{ and } C(t_\mathit{max}^{T^C}) = x\\
			\emptyset & \text{ else,}
		\end{cases} \\
		\llbracket \lozenge \lambda, C, t \rrbracket &=
		\begin{cases}
			\mathcal{C}_\lozenge(\lambda, C, t) & \text{if } t_\mathit{max}^{T^C} + t \leq 20\\
			\emptyset & \text{else,}
		\end{cases} \\
		\llbracket \lambda_1 \vee \lambda_2, C, t \rrbracket &= \llbracket \lambda_1, C, t \rrbracket \cup \llbracket \lambda_2, C, t \rrbracket \text{, and} \\
		\llbracket \lambda_1 \wedge \lambda_2, C, t \rrbracket &= \llbracket \lambda_1, C, t \rrbracket \cap \llbracket \lambda_2, C, t \rrbracket \text{,} \\
	\end{align*}
	with $\mathcal{C}_\lozenge(\lambda, C, t) \coloneq \{ C' \in \mathcal{C}^{[0,t_\mathit{max}^{T^C} + t]} \mid C \subseteq C' \text{ and } \exists t' \leq t\colon C'_{| \leq t_\mathit{max}^{T^C} + t'} \in \llbracket \lambda, C, t' \rrbracket \}$.
	Note that the first line constraints the logic to define only concrete scenarios with time domain $[0,20]$.
	Now, consider 
	$$\lambda_\mathit{ex} = (-50,100,10,-5) \wedge \lozenge ((150,0,10,-5) \vee (0,0,0,0))\text{.}$$
	Let us assume that we do not have any constraints in the world model, i.e.\ $A = \lambda_\mathit{ex}$.
	It is easy to see that the concrete scenario $C$ defined in \autoref{ex:concrete-scenario} is encompassed by $A$, i.e.\ $C \in \mathcal{C}(A)$.
	Moreover, any scenario where the actor comes to a stop at $(0,0)$ is also included.
\end{example}

\section{Comparison of Scenario Formalisms}
\label{sec:comparison}

Based on the formalizations of Section \ref{sec:formalization}, we now conduct a comparison between logical and abstract scenarios regarding their properties and characteristic use cases observed in practice.
Specifically, we ask:
\begin{enumerate}
	\item Expressiveness (Section \ref{subsec:comparsion_expressiveness}): Can we express sets of concrete scenarios in one framework that is not describable in the other?
	\item Specification complexity (Section \ref{subsec:comparsion_specification_complexity}): How high are efforts for specifying scenarios in both frameworks?
	\item Sampling (Section \ref{subsec:comparison_sampling}): How can we generate concrete scenarios from a given logical resp. abstract scenario?
	\item Monitoring (Section \ref{subsec:comparison_monitoring_iia}): How can we decide whether a given concrete scenarios belongs to some logical resp. abstract scenario?
\end{enumerate}
Note that there are several other dimensions worthy of discussion, which, for brevity, are only briefly discussed in Section \ref{sec:discussion}.

\subsection{Expressiveness}
\label{subsec:comparsion_expressiveness}

We first consider expressiveness, which intuitively asks whether we can specify an equivalent logical scenario for any given abstract scenario, and vice versa.
We will see that the former holds, whereas the latter is, in our general framework, not possible.

Let us first examine the latter, for which we exploit that scenario logics encode trees instead of linear temporal evolutions (as dynamic systems do).
In short, there is an abstract scenario that can represent the set of all functions $f\colon \mathbb{R} \to \{0,1\}$ in its encoded tree, which is not possible using the real parameter space of logical scenarios.

\begin{theorem}
	\label{thm:expressiveness1}
	There is an abstract scenario $A$ for which there is no logical scenario $L$ s.t. $\mathcal{C}(A) = \mathit{Im}(L)$.
\end{theorem}

\begin{proof}[Sketch]
	Consider a scenario logic over a time domain $T \subseteq \mathbb{R}$ that has a bijection into $\mathbb{R}$ which allows defining an abstract scenario $A$ that branches at each time into two children.
	A single path through this tree can be described by a function $f\colon \mathbb{R} \to \{0,1\}$, and the set of all paths is equivalent to the set of all such functions $f$, which has a cardinality equivalent to $|2^\mathbb{R}|$ (which is the beth number $\beth_2$).
	But, logical scenarios are restricted to a finite number of real intervals, for each of which a bijection to $\mathbb{R}$ can be given. Hence, scenario sets of cardinality at most $\beth_1$ can be represented using finitely many logical scenarios.
	The difference in expressiveness follows from the difference between $\beth_1$ and $\beth_2$, as there is no bijection between $\mathbb{R}$ and $2^\mathbb{R}$. \qed
\end{proof}

If we consider only discrete time, the above argument does not hold:
Here, each tree allows at most to encode paths of the form $f\colon \mathbb{N} \mapsto \mathbb{R}$ (if we allow infinite branching into the reals at each time point), which, however, only enables representing sets of cardinality $|\mathbb{R}^\mathbb{N}|$, which is equal to $\beth_1$.
Note that for logical scenarios such sets can always be encoded in the parameter space $X$, cf.~equation \eqref{eq:logical_scenario}.

The second part of our examination of expressiveness is concerned with the other direction: 
showing that any logical scenario can be encoded in an abstract scenario of a suitable scenario logic.
\begin{theorem}
	\label{thm:expressiveness2}
	Let $\mathcal{M}$ be the set of deterministic models, $S_0$ a starting scene and $X$ a parameter space.
	Then, there exists a scenario logic $\Lambda$ s.t. for any logical scenario $L_\Phi$ with $\Phi \in \mathcal{M}$ over $S_0$ and $X$ there is an abstract scenario $\lambda_\Phi \in \Lambda$ with $\mathit{Im}(L_{\Phi}) = \mathcal{C}(\lambda_\Phi)$.
\end{theorem}
The core idea is to encode the choice of parameter nondeterministically in the first layer of the tree spanned by the scenario logic.
After choosing the starting scene, the scenario logic simply replicates the behavior of the dynamic model (with some special behavior to respect its accepting condition defined by $\mathcal{C}(\lambda)$).
\begin{proof}[Sketch]
	For the sketch, we assume that for each $\Phi \in \mathcal{M}$ and $x \in X$, we have a closed time domain\footnote{We have to consider the unbounded and open cases too, but omit both for brevity.} $T(\Phi(x), S_0(x)) = [0,t_\mathit{max}^{\Phi,x}]$ and set $T = [0, \max_{\Phi \in \mathcal{M}, x \in X} t_\mathit{max}^{\Phi,x}]$.
	Let $\Lambda \coloneq \mathcal{M}$ and $\lambda_\Phi \in \Lambda$.
	Then, the semantics of $\Lambda$ for $C \in \mathcal{C}$, $t \in T$, and $x \in X$ is defined as
	\begin{equation*}
		\llbracket \lambda_\Phi, C, t \rrbracket \coloneq
		\begin{cases}
			\{\alpha_{S_0(x) \mid \leq t'+t}^{\Phi(x)}\} & \text{if } C = \alpha_{S_0(x) \mid \leq t'}^{\Phi(x)} \text{ and } t+t' \leq t_\mathit{max}^{\Phi,x}\\
			\{\alpha_{S_0(x)}^{\Phi(x)}\} & \text{if } C = \alpha_{S_0(x)}^{\Phi(x)}\\
			\emptyset & \text{else}
		\end{cases}
	\end{equation*}
	where $\alpha_{S_0(x) \mid \leq t}^{\Phi(x)}$ denotes $\alpha_{S_0(x)}^{\Phi(x)}$ restricted to $[0,t]$.
	Note that $\alpha_{S_0(x) \mid \leq 0}^{\Phi(x)} = S_0(x)$.
	For any logical scenario $L_\Phi$ with $\Phi \in \mathcal{M}$ and start scene $S_0$ it holds that $\mathit{Im}(L_\Phi) = \mathcal{C}(\lambda_\Phi)$.
	To show this, let $\alpha_{S_0(x)}^{\Phi(x)} \in \mathit{Im}(L_\Phi)$.
	By construction of $\llbracket \cdot, \cdot, \cdot \rrbracket$, for any $\alpha_{S_0(x)}^{\Phi(x)}$ it holds that $\llbracket \lambda_\Phi, \alpha_{S_0(x)}^{\Phi(x)}, t \rrbracket = \{\alpha_{S_0(x)}^{\Phi(x)}\}$ and, by \autoref{def:scenario-logic} of $\mathcal{C}(\lambda_\Phi)$, $\alpha_{S_0(x)}^{\Phi(x)} \in \mathcal{C}(\lambda_\Phi)$.
	Moreover, since we consider only bounded concrete scenarios, $\mathcal{C}(\lambda_\Phi) = \{ C \in \mathcal{C} \mid \forall t \in T\colon \llbracket \lambda_\Phi, C, t \rrbracket = \{ C \} \}$.
	By this and as $\llbracket \lambda_\Phi, C, t \rrbracket = \{ C \}$ can only be satisfied if $C$ has the form $\alpha_{S_0(x)}^{\Phi(x)}$, 
	it follows that $C \in \mathit{Im}(L_\Phi)$.
\end{proof}

The construction of the scenario logic in \autoref{thm:expressiveness2} illustrates the tree structure of abstract scenarios well, as visualized in \autoref{fig:logic_tree}.
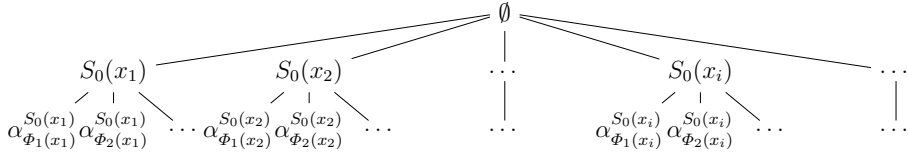
\begin{figure}[htb!]
	\centering
	\vspace{-0.75cm}
	\scalebox{0.78}{

\begin{tikzpicture}[
	level distance=2cm,level 1/.style={sibling distance=3.3cm},
	level 2/.style={sibling distance=1.2cm}, 
	level 3/.style={sibling distance=2.5cm}, 
	level distance=1.0cm,]
	\begin{scope}[scale=1]
		
¸	
	
		\node { {$\emptyset$} }
		child { node {$S_0(x_1)$}
		child { node {$\alpha_{\Phi_1(x_1)}^{S_0(x_1)}$}}
		child { node {$\alpha_{\Phi_2(x_1)}^{S_0(x_1)}$}}
		child { node {$\dots$} } }
		child { node {$S_0(x_2)$}
		child { node {$\alpha_{\Phi_1(x_2)}^{S_0(x_2)}$}}
		child { node {$\alpha_{\Phi_2(x_2)}^{S_0(x_2)}$}}
		child { node {$\dots$} }}
		child { node {$\dots$}
		child { node {$\dots$} } }
		child { node {$S_0(x_i)$}
		child { node {$\alpha_{\Phi_1(x_i)}^{S_0(x_i)}$}}
		child { node {$\alpha_{\Phi_2(x_i)}^{S_0(x_i)}$}}
		child { node {$\dots$} }}
		child { node {$\dots$} 
		child { node {$\dots$} }} ;¸

\end{scope}
	\end{tikzpicture}
	}
	\vspace{-0.5cm}
	\caption{Tree structure of the logic from \autoref{thm:expressiveness2} encoding logical scenarios.}
	\label{fig:logic_tree}
\end{figure}

From \autoref{thm:expressiveness1} and \autoref{thm:expressiveness2} it follows that abstract scenarios are more expressive than logical scenarios, i.e., they subsume logical scenarios and are able to express more properties.
Note, however, that the latter is a rather theoretical consideration based on cardinalities likely not encountered in practice.
It is probable that for all practical purposes, both formalisms can be considered equally expressive. 
However, as we will show in the following, abstract scenarios might deliver a more compact representation since they allow representing trees instead of sets of linearly evolving concrete scenarios, which can, in practice, reduce specification efforts drastically.

\subsection{Specification Complexity}
\label{subsec:comparsion_specification_complexity}

Specification complexity is the effort required to specify scenarios that represent certain sets of concrete scenarios.
Here, we find that a tree structure allows us to compress large sets of concrete scenarios, which logical scenarios must explicitly encode as parameters.

\begin{theorem}
	\label{thm:specification_complexity}
	Abstract scenarios can have a lower specification complexity than logical scenarios, i.e., there is a scenario logic with a formula of constant size that defines $2^n$ many concrete scenarios of length $n \in \mathbb{N}$ s.t. there exist only logical scenarios $L$ of size $\mathcal{O}(2^n)$ that satisfy $\text{Im}(L) = \mathcal{C}(A)$, where the size of $L$ is measured by the size of its parameter space and its number of models.
\end{theorem}

\begin{proof}[Sketch]
	Consider a scenario logic over finite, discrete scenarios of length $n \in \mathbb{N}$ that branches each concrete scenario $C$ into two successor scenarios, i.e., $\llbracket \mathsf{true}, C, 1 \rrbracket = \{ C \oplus 0, C \oplus 1 \}$ if $|C| \leq n$ else $\{C\}$, where $C \oplus n$ denotes the concatenation of $n$ to $C$.
	Essentially, the only formula $\mathsf{true}$ of the scenario logic encodes exponentially many models, i.e., $|\mathcal{C}(\mathsf{true})| = 2^{n}$.
	Since any deterministic model can only represent a unique $C \in \mathcal{C}(\mathsf{true})$, we either require $2^{n}$ models or using a parameterized model $\varphi_k$ for a parameter $k \in 2^{n}$.
	The latter holds as the deterministic model must be explicitly informed about its branching choice at each tree node, requiring $n$ bits.
	In both cases, the number of models or the parameter space is exponential in $n$.
\end{proof}

While these considerations are rather theoretical, we provide a detailed example of an abstract scenario specifying a factorially sized set of concrete scenarios in the \hyperref[appendix]{Appendix}.
As shown in Section \ref{subsec:comparison_sampling}, this possible advantage in specification complexity can come at the expense of higher efforts in other applications, such as sampling.

\subsection{Sampling}
\label{subsec:comparison_sampling}
Now we compare the capabilities and complexities of the two scenario formalisms regarding the use case \emph{play-out} consisting of the stages \emph{sampling} and \emph{execution}.
We can describe play-out as:
\begin{itemize}
       \item Given a logical scenario $L_{\Phi}$ with $\Phi \in [\mathcal{M}]^{<\omega}$ find a concrete scenario $C$ with $C \in Im(L_{\Phi})$, or
       \item given an abstract scenario $A$ find a concrete scenario $C \in \mathcal{C}(A)$.
\end{itemize}
Note that it is often required to compute a large number of concrete scenarios from a given logical or abstract scenario and not just one.

\subsubsection{Logical Scenarios}
\label{subsec:sampling_logical}

Sampling a logical scenario $L_{\Phi}$ is rather straight forward by definition.
To ensure measurability, we assume a fixed set of models $\Phi \in [\mathcal{M}]^{<\omega}$.
For an arbitrary probability space $(\Omega, \mathcal{F}, P)$, we define a random variable $\zeta : \Omega \rightarrow X$, where $X$ is the parameter space of the logical scenario $L_\Phi$.
Assuming the measurability of $L_\Phi$, we can concatenate the two concepts, creating a random function
\begin{align*}
	L_\Phi \circ \zeta \colon (\Omega, \mathcal{F}, P) \rightarrow (\mathcal{C},\mathcal{B}(\mathcal{C})),\,
	\omega \mapsto L_\Phi(\zeta(\omega)) = \alpha_{S_0(\zeta(\omega))}^{\Phi(\zeta(\omega))}\,.
\end{align*}
Thereby, each realization of the random function $L_\Phi \circ \zeta$ is a concrete scenario and the distribution of the concrete scenarios is the push-forward of $P$ under $L_\Phi \circ \zeta$, i.e. $(L_{\Phi} \circ \zeta)_\ast(P)(B) = P((L_{\Phi} \circ \zeta)^{-1}(B)), B \in \mathcal{B}(\mathcal{C})$.
When discussing logical scenarios, however, one is usually more interested in the distribution on the parameter space, which is given in this notation as $\zeta_\ast(P)(B)$ for $B \in \mathcal{B}(X)$. 

\emph{Parameter distribution of a logical scenario} now refers to fixating the distribution $P$ under the push-forward.
This indicates that the initial probability space $(\Omega,\mathcal{F},P)$ is not as relevant as long as the push-forward of $P$ coincides with probability distribution chosen on the parameter space of the logical scenario.
Through this process, we may define parameter distributions for arbitrary parameter spaces, independent of the logical scenario they may belong to in practice.
Afterwards, sampling from the logical scenario simplifies to evaluating the random variable $\zeta$ based on the induced distribution and evaluating the function $L_\Phi$ based on the obtained parameter values.

\subsubsection{Abstract Scenarios}

Sampling from an abstract scenario $A$ is the task of finding $C \in \mathcal{C}(A)$.
This is done by solving the constraint system associated with $A$, commonly known as '(finite) model finding', which may be undecidable for some scenario logics (cf. MLSL).
Thus, to enable a meaningful sampling process, abstract scenarios are usually defined using a decidable (fragment of a) logic.
An example are SMT solver, which are used for TSCs \cite{becker2022simulation}, BTC ES Traffic Phases \cite{eggers2018constraint}, and most likely for M-SDL, as indicated by a patent application \cite{pidan2022techniques}.
An orthogonal approach is using rejection sampling from probability distributions, as done by Scenic \cite{fremont2019scenic}, however, completeness can not be guaranteed (drawn samples might always violate the scene constraints).
For OpenSCENARIO DSL, sampling techniques are left to the tool vendor, whose implementations are yet to emerge.

\subsubsection{Comparing Sampling Effort}
As practical applications confirm, the process of sampling from $\zeta$ is rather simple in most cases, where it usually only takes seconds to generate thousands of pre-execution concrete scenarios. 
However, when comparing the sampled output of logical and abstract scenarios, one has to assume a comparable level of concretization.
Since abstract scenarios skip pre-execution concrete scenarios, one has to also consider the time spent on execution (which, however, can be parallelized).
Sampling abstract scenarios is heavily influenced by the expressiveness of the scenario logic, which limits the representable models and available solver technology.

\subsection{Monitoring}
\label{subsec:comparison_monitoring_iia}

\subsubsection{Logical Scenarios}
\label{subsubsec:comparison_monitoring_iia_logical}

As we have presented the logical scenarios in Section \ref{subsec:construction_logical_scenarios}, there are essentially two cases to consider.
Let $L_\mathcal{M}$ be a logical scenario.
We are concerned with the question
\begin{itemize}
	\item[1)] whether there are deterministic models $\Phi \in [\mathcal{M}]^{<\omega}$ and $x \in X$ s.t. for $C \in \mathcal{C}$ we have $C = L_\Phi(x)$ and thus $x \in L_\Phi^{-1}(C)$, or
	\item[2)] whether for a given model $\Phi \in [\mathcal{M}]^{<\omega}$ we can find $x \in X$ such that $C = L_\Phi(x)$ and thus $x \in L_\Phi^{-1}(C)$.
\end{itemize}
In both of these cases we can ask for the existence of the quantities (monitoring) as well as their actual value (inverse image analysis).

In practice, the first option combined with the inverse image analysis use case would be the one of enabling the re-simulation of observed trajectories with possible intervention based on the actual model dynamics.
For the monitoring use case, it would answer the question of scenario classification, whereby, however, due to the imperative nature of logical scenarios we may not classify by specific situations occuring in the scenario, but more so by starting conditions and whether they conform to a certain physical reality.

The interpretation of the second option would be similar, yet limited to a specific model, e.g. are the models we are currently using able to produce a certain concrete scenario (monitoring) and if so, with which parameters $x \in X$ may we reproduce it.

For example, consider scenes that are limited to the $x$-position of the ego vehicle.
The concrete scenario is the map $C \colon [0, 10] \rightarrow [0, 20], t \mapsto 2t$.
The logical scenario contains a single model $\varphi(x) \colon t \rightarrow xt$, a starting scene $S_0 = 0$ and the parameter space $X = [1, 3]$.
In this simple example one may directly tell that $C = L_\varphi(2)$ with $2 \in X$, thus we indeed have $C \in Im(L_\varphi)$.
However, in cases where the model is not known, the solution to this problem may become as complex as solving a non-linear partial differential equation.

\subsubsection{Abstract Scenarios}
\label{subsubsec:comparison_monitoring_iia_abstract}

Since abstract scenarios are inherently logic-based, they lend themselves easily to monitoring, a problem that has been well-established in the runtime verification community \cite{leucker2009brief}.
For brevity, we only briefly recapitulate the fundamentals of monitoring in temporal logics in the terminology employed in this work.

In its simplest form, the monitoring problems asks, for a given bounded concrete scenario $C$ and an abstract scenario $A$, whether $C \in \mathcal{C}(A)$ (also called the word problem).
Under the mild assumption that a single concrete scenario can be encoded by a formula, we can reduce this problem to checking satisfiability of abstract scenarios: $C \in \mathcal{C}(\lambda)$ iff $\mathcal{C}(\lambda_C) \subseteq \mathcal{C}(\lambda)$ iff $\mathcal{C}(\lambda_C \wedge \lambda) \neq \emptyset$, where $\lambda_C$ is defined suitably s.t. $\mathcal{C}(\lambda_C) = \{ C \}$. 
However, in many settings, the word problem is computationally less complex than satisfiability. 
For example, solving the word problem in LTL over finite words can be done in AC\textsuperscript{2} \cite{kuhtz2009ltl}, whereas satisfiability is PSpace-complete \cite{de2013linear}.

Note that above definition assumes bounded concrete scenarios in $\mathcal{C}(A)$. 
If the scenario logic operates on infinite words, the monitoring problem can be re-defined as asking whether for all infinite extensions $C'$ of $C$ it can be guaranteed that $C' \in \mathcal{C}(A)$.
This prefix problem can be computationally hard, e.g.\ PSpace-complete for LTL over infinite words, which is as hard as checking satisfiability \cite{bauer2013propositional}.

Besides the word and prefix problems, monitoring can be extended to monitors returning true, false, or unknown, depending on whether it is guaranteed that any extension $C'$ of $C$ must be in $\mathcal{C}(A)$, can never be in in $\mathcal{C}(A)$, or neither can be ensured, respectively.

\section{A Practioner's Summary}
\label{sec:discussion}


It remains to transfer above discussions on the properties and use cases of logical and abstract scenarios to practice.
In the end, the scenario-based practitioner is tasked with selecting a suitable formalism.
For this, we briefly summarize our findings in \autoref{tab:scenario_formalisms_characteristics}.

\begin{table}[htb!]
	\centering
	\caption{Characteristics and applications of the considered scenario formalisms: logical and abstract scenarios. \emph{Italicized} properties were discussed in Section \ref{sec:comparison}.}
	\label{tab:scenario_formalisms_characteristics}
	\renewcommand{\arraystretch}{1.2}
	\begin{tabular}{L{4cm}L{4cm}L{4cm}}
		\toprule
		\textbf{Characteristics} & \textbf{Logical Scenarios} & \textbf{Abstract Scenarios} \\
		\midrule
		\emph{Approach} & Imperative (Command-based) & Declarative (Constraint-based) \\
		\emph{Expressiveness} & Theoretically Lower; Practically Equivalent & Theoretically Higher; Practically Equivalent\\
		\emph{Specification Complexity} & Higher & Lower \\
		Description of Degrees of Freedom & Explicit & Implicit \\
		Dimensionality & Finite & Unbounded \\
		\midrule 
		\textbf{Applications} &  \textbf{Logical Scenarios} & \textbf{Abstract Scenarios}  \\
		\midrule
		Specifying & Data-based & Expert-based \\
		\emph{Sampling} & Easier (evaluation of parameter ranges) & Harder (sampling temporal formulae) \\
		\emph{Monitoring} & Unsuitable & Suitable \\
		Reasoning & Unsuitable & Suitable \\
		\bottomrule
	\end{tabular}
\end{table}
We highlight again that the main difference between abstract and logical scenarios is their declarative and imperative approach, which becomes visible in our formal framework of Section \ref{sec:formalization}.
This translates to theoretical differences of expressiveness (Section \ref{subsec:comparsion_expressiveness}), which we found, to all likelihood, to be of little practical relevance.
However, abstract scenarios can reduce efforts during specification (Section \ref{subsec:comparsion_specification_complexity}), as we do not require to explicate all possible options but rather encode them in a constraint.

Two characteristics were not discussed in Section \ref{sec:comparison}.
The first is the introduction of degrees of freedom to concrete scenarios, which, for both formalisms, is fundamentally influenced by their imperative and declarative natures.
For logical scenarios, one has to explicitly formulate the degrees of freedom as parameters.
With abstract scenarios, those are kept implicit -- any variable of the world model might be a degree of freedom, depending on whether the formula fixes the variable.
Second, in practical settings, the dimension of the space of all scenes $\mathcal{S}$ is finite for logical scenarios -- e.g.\ $\mathbb{R}^4$ in our example.
For scenario logics, the scene space may not be fixed to a certain dimension.
For example, in monitoring, concrete scenarios with scenes of arbitrary dimensions must be matched, and thus the logic cannot put an a-priori bound on the scenes' dimensionality.

These properties influence how well the formalisms are suited for scenario-based applications.
As a first step, scenarios must be specified: 
Logical scenarios are often constructed bottom-up, generalizing from observed concrete scenarios (data).
Abstract scenarios, on the other hand, are defined by experts and often encode general knowledge (such as traffic rules).
As the history of Section \ref{sec:history} showed, sampling was not an early use case for abstract scenarios -- in fact, sampling expressive constraint systems may even be undecideable (Section \ref{subsec:comparison_sampling}).
Logical scenarios, however, were designed with sampling in mind, which reduces to drawing values from its parameter space.
Monitoring (Section \ref{subsec:comparison_monitoring_iia}) was one of the early applications of abstract scenarios, as temporal logics are often exploited for runtime verification, which is often easier than sampling abstract scenarios.
It is not straightforward to extend monitoring to logical scenarios.
Finally, deductive reasoning has been naturally applied to abstract scenarios, especially if they specify requirements, e.g.\ for consistency \cite{becker2020partial}.
Again, deduction on logical scenarios might be difficult: First, a description language for logical scenarios would need a formal semantics (which is often not the case) and second, this amounts to analyzing dynamical systems.

\section{Conclusion and Future Work}
\label{sec:conclusion}

This work revisits the scenario qualifications concrete, logical and abstract commonly used for automated driving and puts them on a formal basis. 
Based thereon, we conducted a comparison between logical and abstract scenarios regarding their expressiveness, complexity of specification, and for sampling and monitoring. 
Future work includes detailed proofs for the theoretical results, a comparison of the formalisms in practice and regarding additional properties, and industrial questions such as the protection of intellectual properties in the world model of abstract scenarios.

\bibliographystyle{elsarticle-num}
\bibliography{Literature}

\clearpage
\appendix
\section*{Appendix}
\label{appendix}

\paragraph{Example of a TSC-based Abstract Scenario}
Beyond the theoretical considerations made in Section \ref{subsec:comparsion_specification_complexity} using our generic formalizations, let us consider an example using a fixed scenario logic. We use TSCs to compresses a factorial-sized set of concrete scenarios into a linearly sized abstract scenario. 
As already mentioned in Section \ref{subsec:history_formal}, TSCs are a visual, yet formal, declarative specification language for abstract scenarios.
We abstain from explaining the semantics of TSCs here, but rather consider a hands-on example that contextualizes the statement of \autoref{thm:specification_complexity}. To understand TSCs it is most important how the objects are positioned relatively to each other. 

\begin{figure}[htb!]
	\centering
	\includegraphics[width=\linewidth]{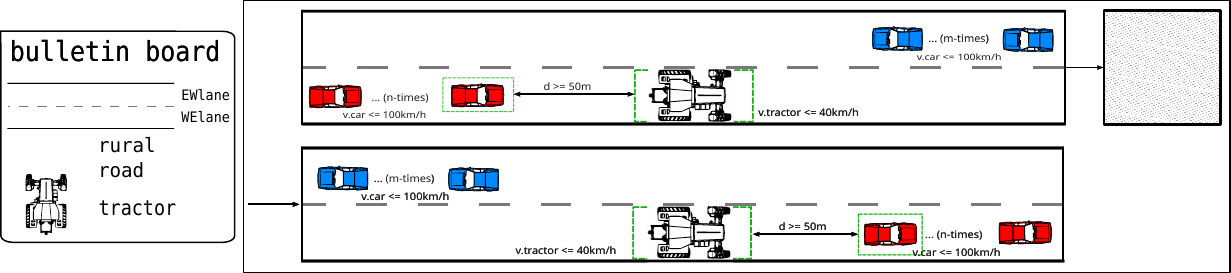}
	\vspace{-0.5cm}
	\caption{Example of an abstract scenario specified as Traffic Sequence Chart (TSC): on a rural road with one lane in each direction, a slow tractor is overtaken by the $n$ red passenger cars with oncoming traffic in form of $m$ blue passenger cars.}
	\label{fig:tsc_example_rural}
\end{figure}

In the example TSC of \autoref{fig:tsc_example_rural}, the combinatorial explosion happens due to spatial combinations on scene level rather on the temporal level as in \autoref{thm:specification_complexity}, highlighting yet another difference in the low specification complexity of abstract scenarios.
\autoref{fig:tsc_example_rural} describes the following abstract scenario $A_{rural}$:
The scenario is located on a rural road with one lane in each direction, called east-west lane (EWlane) and (WElane) as specified in the bulletin board. There are three continuous sequential time phases (\emph{invariant nodes}) in the scenario. In the first invariant node, a tractor jams the WElane due to its speed which is constrained to $v_{\text{tractor}} \le 40\text{km/h}$.
There are $n \in \mathbb N$ red passenger cars in the WElane behind the tractor with distance  $d(\text{red-cars},\text{tractor}) \ge 50\text{m}$  and $m \in \mathbb{N}$ blue passenger cars on the EWlane in front of the tractor.
In the third invariant node the red passenger cars are again on the WElane but constrained to be in front of the tractor, again with $d(\text{red-car},\text{tractor}) \ge 50$ m. Additionally, the blue passenger cars are still on the EWlane but have now passed the tractor. The speed of the passenger cars is constrained by  $v_{\text{car}} \le 100$ km/h.
The second invariant node is \emph{empty}, meaning the happenings between the first and the third invariant node are unconstrained.
Therefore, in any concrete scenario satisfying these constraints, the red cars will overtake the tractor and that the blue cars - being oncoming traffic - interfere with the red cars' overtaking maneuvers.
Note that TSCs allow existentially quantifying objects within the scope of a single invariant (by not fixing them in the bulletin board). By this, the presence of a e.g. a red car simply implies the existence of \emph{any} (and not some specific) red car, and that car might be a different one in a subsequent invariant. Therefore, their ordering in the first invariant node is not necessarily the same as in the third invariant node.

This abstract scenario's sparse constraints leave open many possibilities, for example:
\begin{itemize}
	\item the order of the red cars overtaking the tractor is not constrained
	\item the combinatorics of how many blue cars pass in front of a given red car before overtaking is not constrained
	\item the order of the red/blue passenger cars in the first and in the last invariant node is not constrained (i.e.~there overtaking/fallback maneuvers be happening among the red/blue passengers care before or after the overtaking the tractor)
	\item Red passenger cars could disengage from overtaking the tractor and fall back behind $n' \le n$ of the red cars 
	\item the tractor could come to a stop, i.e.~$v_{\text{tractor}} = 0\text{ km/h}$ and turn into a static obstruction (or even reverse with $v_{\text{tractor}} < 0 \text{ km/h}$)
\end{itemize}

Let us now fix $m, n \in \mathbb N$ and count the possibilities of the triple ('ordering of red cars' overtaking maneuvers', 'combinations of blue cars passing tractor before red cars', 'ordering of red cars after overtaking') which consists of three independent events. As we assume $n \in \mathbb{N}$ red cars, there are $n! = 1 \cdot 2 \cdot \dots \cdot n$ possibilities how the red cars can perform the overtaking and also $n!$ possible orderings after the overtaking. Regarding the oncoming traffic in form of $m \in \mathbb{N}$ blue cars, every red car $k \in \{1,\dots,n\}$ has the option to let $r_k$ blue cars pass before them with $r_0 + r_1 + \dots + r_n = m$, $0 \le r_k \le m$ and $r_0$ is the number of blue cars that no red car waited for. Combinatorially, this corresponds to the number of weak compositions of $m$ into $n+1$ parts which amounts to $\binom{m+n}{n}$ possibilities. Therefore, in total we have a lower bound of
$(n!)^{2} \cdot \binom{m+n}{n} \le |C(A_{\text{rural}})|$
for the set of concrete scenarios. 
Of course, the actual cardinality of $\mathcal{C}(A_{\text{rural}})$ is not finite as variations of parameters such as positions and velocities are from real intervals. However, these combinatorial aspects are on maneuver-level and cannot be easily dealt with. For $n=3$ red cars and $m=2$ blue cars we already have $(3!)^{2} \cdot \binom{5}{3} = 6^{2}\cdot10 = 360$ possibilities. The declarative nature of TSCs allows for these combinatorics to stay implicit on the specification level (rather than explicit). This does not mean that the combinatorial complexity has vanished, but it is delegated to subsequent applications such as sampling \autoref{subsec:comparison_sampling}.

\paragraph{Modeling the Example as Logical Scenario}

In theory, depending on the models $\Phi$, it might be possible to specify this example abstract scenario imperatively as a logical scenario $L_{\Phi}$ -- using finitely many parameters.
For this, the combinatorics -- which are implicit in the declarative description -- of the abstract scenario of \autoref{fig:tsc_example_rural} would have to be made explicit.
This means encoding in a parameter space $X_{A_{\text{rural}}}$ all the possible parametrizations $x \in X_{A_{\text{rural}}}$ of models $\Phi^{x}$ and of the starting scene $\mathcal{S}_0^{x}$ such that the same set of concrete scenarios is described, i.e.\ $L_{\Phi}(X_{A_{\text{rural}}}) = \mathcal{C}(A_{\text{rural}})$.
The first problem arises with parametrizing the models $\Phi$ such that all these different combinations are even possible. The second problem is the sheer number of possibilities which is factorial in $n$ and $m$. Even for small $n$ and $m$ (e.g.~$(3,2)$), the specification as a logical scenario is practically impossible.
A low-effort approximation could be the passing of a sophisticated aggressiveness parameter to the models controlling the red cars to maybe cover a part of the combinatorics.

\paragraph{Other Modeling Options}
A declarative scenario specification is not the only solution to the combinatorial complexity of our example. One possibility is to use process algebras, e.g.\ communicating sequential processes (CSP), to model the scenario as interaction in a concurrent system. Each car in the example scenario could be modeled as a separate process and the interleaving operator of CSP would implicitly specify all the different possibilities for the interaction without enumerating them explicitly.

\end{document}